\documentclass[journal, onecolumn]{IEEEtran}
\usepackage{times}  
\usepackage{helvet}  
\usepackage{courier}  
\usepackage{url}  
\usepackage{graphicx}  
\usepackage{amsmath,mathtools,amsthm}
\usepackage{amssymb}
\usepackage{amstext}
\usepackage{tikz}
\usepackage{mathrsfs} 
\usepackage{subfigure}
\usepackage{color}
\usepackage[linesnumbered,ruled]{algorithm2e}

\usepackage{multirow}

\newcommand{\subseqs}{\mathbf{S}}

\newcommand{\ranking}{\mathbf{r}}
\newcommand{\rankings}{\mathcal{R}}

\newtheorem{definition}{Definition}
\newtheorem{theorem}{Theorem}

\newtheorem{remark}{Remark}

\newtheorem{example}{Example}

\ifCLASSINFOpdf
\else
\fi
\begin{document}
%
\title{Quantifying consensus of rankings based on $q$-support patterns}
%
%
%

\author{Zhengui Xue, Zhiwei Lin, Hui Wang,
        and Sally McClean 
\thanks{The authors are with the School of Computing, University of Ulster, United Kingdom. e-mail: zhenguixue@gmail.com, z.lin@ulster.ac.uk, h.wang@ulster.ac.uk, si.mcclean@ulster.ac.uk}
}

\maketitle

\begin{abstract}
Rankings, representing preferences over a set of candidates, are widely used in many information systems, e.g., group decision making and information retrieval.  It is of great importance to evaluate the consensus of the obtained rankings from multiple agents. 
An overall measure of the consensus degree 
provides an insight into the ranking data. Moreover, it could provide a quantitative indicator for consensus comparison  between groups and further improvement of a ranking system.  Existing studies are insufficient in assessing the overall consensus of a ranking set. They did not provide an evaluation of the consensus degree of preference patterns in most rankings. In this paper, a novel consensus quantifying approach, without the need for any correlation or distance functions as in existing studies of consensus, is proposed based on a  concept of $q$-support patterns of rankings.   The $q$-support patterns represent the commonality embedded  in a set of rankings. A method for detecting outliers in a set of rankings is naturally derived from the proposed consensus quantifying approach. Experimental studies are conducted to demonstrate the effectiveness of the proposed approach.
\end{abstract}


%
\IEEEpeerreviewmaketitle

\section{Introduction}
%
%
%
%

\IEEEPARstart{E}{xtensive} studies have been carried out in social science to measure group cohesion, in order to gain insight into the factors affecting group cohesion and further promote higher group consistency (see, e.g., \cite{hogg1993group,carron2000cohesion,salas2015measuring,chiniara2018servant}). 
%
In artificial intelligence, rankings have been widely used to represent the preferences of agents (humans or systems) over a set of candidates in many information systems, such as group decision making \cite{kim1999interactive,qin2015multi,zhu2014deriving} 
and information retrieval \cite{hotho2006information,liu2009learning,poshyvanyk2007feature}. 
%
It is important to evaluate the degree to which the rankings obtained by different agents agree, as it would help to understand the obtained rankings.
%
However, to the best of our knowledge, there are only a few existing studies on the evaluation of the overall consensus degree for a set of rankings.
%
Quantifying the {\em consensus} of the obtained rankings is to provide an accurate measure about the overall agreement. It is also a quantitative indicator for comparing consensus between groups (e.g., two sets of rankings) \cite{Bosch2005}  or for further improving the ranking systems. For example, in group decision making, if  the consensus score is extremely low, it is necessary for the experts to adjust their rankings in order to reach an agreement \cite{kim1999interactive}.

Rank {\em correlation} or {\em distance} functions, such as the Kendall's $\tau$ \cite{kendall1938new} and the Spearman's $\rho$ \cite{spearman1904proof}, have been proposed to measure the correlation or  disagreement of two rankings, it is however difficult to use them to quantify the level of consensus for a set of rankings (more than two rankings) in a full picture. The Kendall's $\tau$ measures the correlation of two rankings by considering their concordant and discordant pairs, and the Spearman's $\rho$ evaluates the rank correlation by taking into account the positions of the items in two rankings. The Kemeny distance  \cite{kemeny1959mathematics} is extended to measure  the pairs of disagreed preferences in two rankings. 
A related concept is {\em cohesiveness}, which is used interchangeably for {\em consensus}.    Cohesiveness measures the similarity of preferences in a group. The most common existing approaches to measuring the similarity of preferences in a set of rankings need to calculate the similarity for each pair of rankings based on correlation functions and then aggregate the obtained results \cite{alcalde2013measuring}.  {\em Diversity} and  {\em cohesiveness} are considered as two opposite concepts of rankings in social choice theory \cite{karpov2017preference}. Research was carried out  to measure the diversity of a ranking set based on distance functions (see \cite{garcia2010consensus}) . However, these studies are far from sufficient in evaluating the overall consensus of a ranking set. In reality, it is often the case that certain preference patterns are embedded in most of the rankings  obtained for a task. The existing work cannot tell the degree to which preferences over candidates are shared by the majority of the rankings. In addition, they did not provide a solution to identifying the majority of rankings  in order to filter irrelevant results in the ranking set, which could play an important role in modern information systems.  For instance, when providing auto-suggestion queries in search engine,   
the suggested terms or queries must be as close to user's search intents as possible, thus it is important to remove the outlier queries resulting in low consensus rankings from the suggestion list in order to provide the users accurate search results.


%
%


This paper studies the consensus degree of a ranking set from a different perspective to provide a full picture on the degree to which a set of rankings mutually agree. This work proposes a novel framework to analyse the consensus of rankings by considering the
common patterns embedded in a ranking set.  A new concept of $q$-support patterns is introduced to represent how common patterns are embedded in rankings, by which the preferences of a group over candidates can be expressed at a subtle and fine-grained level. A pattern is regarded as a $q$-support pattern if it is included by at least $q$ rankings in the ranking set. Thus, a $q$-support pattern represents the partial coverage of the pattern by rankings, where the integer $q$ can be specified as needed when a ranking system is evaluated. The consensus of rankings is quantified based on the number of $q$-support patterns.
Compared with the existing work based on correlation or distance functions, this new approach gives a finer characterization and quantification of the commonalities embedded in the rankings.

The contribution of this paper includes: (1) a new representation of the commonality within a set of rankings -- $q$-support pattern is proposed; (2) a new framework (non-distance or non-correlation) for quantifying consensus with $q$-support patterns is introduced; (3) an efficient algorithm is developed to calculate consensus scores and characterize the set of $q$-support patterns; (4) consensus scores are defined for each ranking to reflect its relationship with the other rankings, which can be used to detect outliers in a ranking set; (5) extensive experiments have been conducted to show the effectiveness and usefulness of the proposed approach.

The rest of the paper is organized as follows. In Section \ref{section: Related work}, related work on the pairwise comparison of rankings and the measure of consensus and diversity of rankings is reviewed. In Section \ref{section: Formulation of consensus quantities},  the $q$-support pattern of rankings is formulated and consensus scores are defined based on it. An algorithm is then introduced to quantify ranking consensus. In Section \ref{section: Detecting outliers}, an outlier detection method is developed. 
In Section \ref{section: Weighted Consensus Quantifying}, weighted consensus scores are defined. Section \ref{section: Experiment} gives experimental studies to evaluate the proposed approach.
Section \ref{section: Conclusion} concludes this paper.

\section{Related work}
\label{section: Related work}
\textbf{Rank correlation and distance functions.} Historically developed by Maurice Kendall in 1938 \cite{kendall1938new}, Kendall's $\tau$ measures the  correlation between two rankings by considering the numbers of
pairwise items ranked in the same orders and in opposite orders. 
Suppose that we consider rankings over candidates $\{\sigma_1,\sigma_2,\cdots,\sigma_n \}$.  A ranking is an ordered list in which items in higher positions are more preferred than items in lower positions.  Let $\pi(\cdot, \cdot)$ be  the position function. The function $\pi(\sigma_i, \ranking_l)$ returns the position of  item $\sigma_i$ in ranking $\ranking_l$. The Kendall's $\tau$ for  two rankings  $\ranking_l$ and $\ranking_z$ is
\begin{eqnarray}
&&\tau(\ranking_l, \ranking_z) = 
\frac{\sum\limits_{\substack{i,j \in\{1,\cdots,n\} \\ i<j}} \text{sgn}(\pi(\sigma_i,  \ranking_l) - \pi(\sigma_j,  \ranking_l))
\text{sgn}(\pi( \sigma_i, \ranking_z) - \pi(\sigma_j,  \ranking_z))}{{n(n-1)}/{2}}.\nonumber
\end{eqnarray}
This coefficient  is in the range $-1 \leq \tau(\ranking_l, \ranking_z) \leq 1$, where value 1 correspons to the case that the two rankings are in the same order and value $-1$ indicates that one ranking is in the reverse order of the other. 

Spearman's $\rho$ proposed by Charles Spearman in 1904 \cite{spearman1904proof} is defined based on the position of each item in two rankings as follows
\begin{equation*}
\rho(\ranking_l, \ranking_z) =  \frac{\sum\limits_{i=1}^n (\pi(\sigma_i, \ranking_l) - \bar{\pi}_l) ( \pi(\sigma_i, \ranking_z)-\bar{\pi}_z)  }{\sqrt{\sum\limits_{i=1}^n  (\pi(\sigma_i, \ranking_l) - \bar{\pi}_l) ^2 \sum\limits_{i=1}^n  (\pi(\sigma_i, \ranking_z) -\bar{\pi}_z) ^2 }},
\end{equation*}
where $\bar{\pi}_l = \frac{1}{n}\sum\limits_{i=1}^n \pi(\sigma_i, \ranking_l)$ and $\bar{\pi}_z= \frac{1}{n}\sum\limits_{i=1}^n \pi(\sigma_i, \ranking_z)$. Similarly, this coefficient satisfies  $-1 \leq \rho(\ranking_l, \ranking_z)  \leq 1$.

These rank correlation functions do not take into account the varying relevance of ranked items in different positions. They are not suitable for evaluating the rankings where  items at the top of a ranking are much more important than those at the bottom 
\cite{fagin2003comparing}. Further studies on weighted rank correlation were carried out extensively based on these two functions  \cite{carterette2009rank,iman1987measure,kumar2010generalized,shieh1998weighted,vigna2015weighted,webber2010similarity,yilmaz2008new}. More reasonable variants of rank correlation functions were also proposed in the literature \cite{etesami2016maximal,hassanzadeh2014axiomatic,henzgen2015weighted,tan2015family}.

Distance metrics have been used to analyze ranking data. One of the most widely used distance functions to measure rankings is the Kemeny distance  \cite{kemeny1959mathematics}. It is defined as  
the sum of pairs where the ranking preferences disagree. One can refer to  \cite{baigent1987preference,nurmi2004comparison,eckert2011distance}
for more information about the commonly used distance metrics.

\textbf{Measuring consensus and diversity of rankings.} In existing studies, consensus and diversity of rankings are typically measured by making pairwise comparisons of the rankings and aggregating the comparison results. Thus, two key issues with these approaches are the utilization of proper comparison metrics and aggregation methods. A consensus measure was first proposed in \cite{Bosch2005} with simple axioms including unanimity, anonymity and neutrality.
Work \cite{garcia2010consensus} improved the study of \cite{Bosch2005} by considering weighted Kemeny distance. Extended work with more reasonable distance metrics was carried out \cite{alcalde2011measuring,garcia2011measuring,hashemi2014measuring,erdamar2014measuring}. In  \cite{karpov2017preference}, a generalization of work  \cite{alcalde2013measuring} was developed with
a geometric mean aggregator and the leximax comparison. 

 

\section{Quantifying consensus with $q$-support patterns}
\label{section: Formulation of consensus quantities}
This section first defines the $q$-support patterns and consensus scores of a ranking set. Then, an algorithm is presented to calculate the consensus scores by utilizing matrices to represent the $q$-support patterns. 

\subsection{$q$-support patterns}
Let $\mathcal{C}=\{\sigma_1,\sigma_2,\cdots,\sigma_n \}$  be a set of $n$ candidates to be ranked. 	 A ranking $\ranking_l=\left(r_{l_1}, r_{l_2}, \cdots, r_{l_m}\right)$ is an ordered list in which item $r_{l_i} \in \mathcal{C}$  is more preferred than item  $r_{l_j} \in \mathcal{C}$ for $i<j$.  Given two items $\sigma_{x}$ and $\sigma_{y} \in \mathcal{C}$, if there exists $i \leq j$ such that
$r_{l_i} = \sigma_{x}$ and $r_{l_j} = \sigma_{y}$, we write $\sigma_x \sigma_y   \sqsubset\ranking_l$; otherwise $\sigma_x \sigma_y \not \sqsubset\ranking_l$. Specially, if $\sigma_{x} = \sigma_{y}$, $\sigma_x\sigma_x   \sqsubset\ranking_l$ simply means that item $\sigma_x$ is included in ranking $\ranking_l$, also written as $\sigma_x  \sqsubset\ranking_l$.

It is usually the case that most of the rankings obtained for a task share certain commonality.  Suppose that there is a set of rankings $\rankings=\{\ranking_1 = (a,b,c,d,e,f), ~ \ranking_2 = (b,a,c,d,e,f), ~\ranking_3 = (a,b,c,e,d,f), ~\ranking_4 = (c, b,d,e,f,g)\}$. It can be seen that item $a$ and the pairwise item $bc$  are common patterns for most of the rankings, but not for all the rankings in $\rankings$ (e.g., $bc\sqsubset \ranking_1, bc\sqsubset \ranking_2, bc\sqsubset\ranking_3$, but  $bc\not\sqsubset\ranking_4$). These patterns, partially included in a set of rankings, show the extend to which the rankings agree. Therefore, it is necessary to consider these patterns to understand the consensus level in a set of rankings.  As such, we define the following $q$-support patterns for a ranking set.
\begin{definition}[$q$-support patterns]
Consider a set of $N$ rankings $\rankings=\{\ranking_1,\ranking_2,\dots, \ranking_N \}$ over candidate set $\mathcal{C}=\{\sigma_1,\sigma_2,\cdots, \sigma_n \}$. For $\sigma_{x}$ and $\sigma_{y} \in \mathcal{C}$, we have the following subset $\rankings'(\sigma_x,\sigma_y) \subseteq\rankings$ 
\begin{equation}
\label{eq:RankingSetIncludeItems}
\rankings'(\sigma_x,\sigma_y)  = \left\{ \ranking_z  | \sigma_x\sigma_y \sqsubset\ranking_z, ~ \ranking_z \in \rankings \right\}.
\end{equation}
Let $ q \in (0, N]$ be an integer. The pattern $\sigma_x\sigma_y$ is a $q$-support of $\rankings$, denoted by $\sigma_x\sigma_y\stackrel{q}{\sqsubset}\rankings$, if the size of $\rankings'(\sigma_x,\sigma_y)$ satisfies $|\rankings'(\sigma_x,\sigma_y) |\geq q$; otherwise $\sigma_x\sigma_y\not\stackrel{q}{\sqsubset}\rankings$. If 
$\sigma_{x} = \sigma_{y}$, $\sigma_x\sigma_x \stackrel{q}\sqsubset\rankings$ indicates that item $\sigma_x$ is a single $q$-support item of $\rankings$, also written as $\sigma_x \stackrel{q}\sqsubset\rankings$.
\end{definition}

The notation $\sigma_x\sigma_y \stackrel{q}\sqsubset\rankings$ means that $\sigma_x\sigma_y$ occurs in at least $q$ rankings in $\rankings$. We use  $\subseqs_1(q)$ and $\subseqs_2(q)$ to respectively denote the set of the single $q$-support items and the set of the pairwise $q$-support patterns, i.e., 
\begin{eqnarray}
\subseqs_1(q) \!&\! = \!&\! \left\{ \sigma_x |\sigma_x \stackrel{q}{\sqsubset}\rankings, ~\sigma_x \in \mathcal{C} \right \}\\
\subseqs_2(q) \!&\! = \!&\! \left\{ \sigma_x\sigma_y |\sigma_x\sigma_y \stackrel{q}{\sqsubset}\rankings, ~\sigma_x \neq \sigma_y, \sigma_x \in \mathcal{C},\sigma_y \in \mathcal{C} \right \}.
\end{eqnarray}
The set $\subseqs_1(q)$ is important in the evaluation of incomplete rankings, where not all the candidates under consideration are ranked in the rankings. It gives the items with more preferences among the candidates, which are ranked in at least $q$ rankings. 
The set  $\subseqs_2(q)$ collects the most preference orders of the items in $\subseqs_1(q)$.


\subsection{Consensus scores}
The $q$-support patterns describe how common patterns are embedded in rankings. This section first defines individual consensus scores for a ranking $\ranking_l \in \rankings$ based on the $q$-support patterns.
Then, the overall consensus scores are introduced for the ranking set $\rankings$. The relative consensus degree that a ranking  $\ranking_l$ shares with the others can be revealed by the individual and the overall consensus scores. In Section \ref{section: Detecting outliers}, it shows that this information can be used in the detecttion of an outlier from a ranking set. 

The following individual consensus scores are defined for a ranking $\ranking_l$. 
\begin{definition}[Individual consensus scores] 
\label{def:IndividualConsensusScore}
For an arbitrary  ranking $\ranking_l=\left(r_{l_1}, r_{l_2}, \cdots, r_{l_m}\right) \in \rankings$,  the sets of the single $q$-support items and the pairwise $q$-support  patterns are defined as
\begin{eqnarray}
\label{eq:setSingleqSupportItems}
\subseqs_1^{\ranking_l}(q) \!&\!=\!&\! \left\{ r_{l_i}  |r_{l_i} \stackrel{q}{\sqsubset}\! \rankings, ~ i, \!\in\! \{1, 2, \cdots, m\}\right\}\\
\label{eq:setqSupportPatterns}
\subseqs_2^{\ranking_l}(q) \!&\!=\!&\! \left\{ r_{l_i}r_{l_j}  |r_{l_i} r_{l_j} \stackrel{q}{\sqsubset}\! \rankings, ~ i,j  \!\in\! \{1, 2, \cdots, m\}, i \! < \! j\right\}.
\end{eqnarray}
The individual consensus scores of $\ranking_l$ are
\begin{eqnarray}
\kappa_1^{\ranking_l}(q)\!\!\!\!\!\!&&\!\! = \frac{1}{N_1^{\ranking_l}  } |\subseqs_1^{\ranking_l}(q)| 
\label{eq:individualK1}\\
\kappa_2^{\ranking_l}(q)\!\!\!\!\!\!&&\!\! = \frac{1}{N_2^{\ranking_l}}   |\subseqs_2^{\ranking_l}(q)|.
\label{eq:individualK2}
\end{eqnarray}
where 
$N_1^{\ranking_l}  = m$ and $N_2^{\ranking_l} = \frac{m(m-1)}{2}$ respectively  represent the number of the ranked items and the pairwise patterns of $\ranking_l$.
\end{definition}

\begin{definition}[Overall consensus scores]
\label{def:OverallConsensusScore}
For a ranking set $\rankings$ with the individual consensus scores defined as Eqs. (\ref{eq:individualK1}) and (\ref{eq:individualK2}),  the overall consensus scores of $\rankings$ are
\begin{eqnarray}
\label{eq:overallK1score}
\bar{\kappa}_1(q) &=&  \frac{1}{N}\sum_{l = 1}^N \kappa_1^{\ranking_l}(q) \\
\label{eq:overallK2score}	
\bar{\kappa}_2(q) &=&  \frac{1}{N}\sum_{l=1}^N \kappa_2^{\ranking_l}(q).
\end{eqnarray}
\end{definition}

The individual consensus scores measure the proportions of the preference patterns of $\ranking_l$ embedded in at least $q$ rankings, where $\kappa_1^{\ranking_l}(q)$ measures the consensus in terms of single $q$-support items and $\kappa_2^{\ranking_l}(q)$ measures the consensus in terms of pairwise $q$-support patterns. The overall consensus scores give the average proportions and they are used to evaluate the consensus degree of a whole ranking set. They have the following property.

\noindent\textbf{Property 1.} The overall consensus scores satisfy
\begin{eqnarray}
&&0 \leq \bar{\kappa}_1(q) \leq 1\\
&&0 \leq \bar{\kappa}_2(q) \leq 1.
\end{eqnarray}
The score $\bar{\kappa}_1(q)=0$  if and only if arbitrary $q$ rankings in $\rankings$ share no common item, and $ \bar{\kappa}_1(q)=1$ if and only if every ranked items of all the rankings is shared by at least $q$ rankings. Similarly, $\bar{\kappa}_2(q)=0$  if and only if  arbitrary $q$ rankings in $\rankings$ share no common pairwise pattern, and $ \bar{\kappa}_2(q)=1$ if and only if every pairwise preference pattern of all the rankings is embedded in at least $q$ rankings. 

%
%
%
%
%
%

\subsection{An efficient algorithm to quantify  consensus}
\label{subsection:Consensus evaluation of rankings}
In this section, a matrix representation is introduced to represent the $q$-support patterns, shown in Theorem \ref{theoremAmatrix}, which implies  an algorithm for calculating the consensus scores.   

\begin{theorem}
\label{theoremAmatrix}
Consider a set of $N$ rankings $\rankings=\{\ranking_1,\ranking_2,\dots,\ranking_N \}$ over candidates  $\mathcal{C}=\{\sigma_1,\sigma_2,\cdots,\sigma_n \}$. For a ranking $\ranking_l = \left(r_{l_1}, r_{l_2}, \cdots, r_{l_m}\right) \in \rankings$ and $\forall \ranking_z = (r_{z_1}, r_{z_2}, \cdots, r_{z_u}) \in \rankings$, with the position function
\begin{equation}
\label{eq:PositionFunction}
\pi\left(r_{l_i}, \ranking_z\right) =
\begin{cases}
0, & \text{if}~ r_{l_i}\not\sqsubset  \ranking_z  \\
p, & \text{if}~ r_{l_i} = r_{z_p}
\end{cases}	
\end{equation}	
and  the Heaviside function 
\begin{equation} \label{eq:Heaviside}
H(x)=
\begin{cases}
1, & \text{if}~ x> 0  \\
0, & \text{otherwise},
\end{cases}
\end{equation} 
we define
\begin{eqnarray}
\label{eq:fFunctionCount}
 f(r_{l_i}, r_{l_j})= \!\!\!\!\!\!\!\!\!\!\!\!&&~\begin{cases}
\sum\limits_{z=1}^{N}{H \left(\pi(r_{l_i}, \ranking_z)\right),~~~~~~~~~~~~~~~~~~~~~~~~~~~~~~  \text{if}~ i =j}\\	
\sum\limits_{z=1}^{N}{ H  \left(\pi(r_{l_j}, \ranking_z) - \pi(r_{l_i}, \ranking_z)\right)  H\left( \pi(r_{l_i}, \ranking_z)\right),   \text{otherwise}}
\end{cases}		
\end{eqnarray}
and  matrix $\mathbf{A}^{\ranking_l}  = \left(A^{\ranking_l}[j,i]\right)\in\mathbb{R} ^{m\times m}$ as {
\begin{equation} \label{eq:a_ij}
A^{\ranking_l}[j,i] = 
\begin{cases}
1,  & \text{if} ~i \leq j~  \text{and}~ f(r_{l_i}, r_{l_j}) \geq q\\
0,  & \text{otherwise}.
\end{cases}
\end{equation}}
Then, we have
\begin{eqnarray}	
\label{eq:k1Calcul}
\kappa_{1}^{\ranking_l}(q) &=&  \frac{1}{N_1^{\ranking_l}} \text{tr}\left(\mathbf{A}^{\ranking_l}\right)\\
\label{eq:k2Calcul}
\kappa_{2}^{\ranking_l}(q) &=&  \frac{1}{N_2^{\ranking_l} } \left( \mathbf{e}^T \mathbf{A}^{\ranking_l} \mathbf{e} - \text{tr}\left(\mathbf{A}^{\ranking_l}\right)\right),
\end{eqnarray}	
where $\mathbf{e} = [1, 1, \cdots, 1]^T$ is an $m$-row vector of all ones. 	
\end{theorem}
\begin{proof}
By Eq. (\ref{eq:PositionFunction}), it can be known that $\pi\left(r_{l_i}, \ranking_z\right)$ gives the position of item $r_{l_i}$ in $\ranking_z$. From the definition of $ f(r_{l_i}, r_{l_j})$, it can be seen that 
	$ f(r_{l_i}, r_{l_j})$ counts the number of rankings $ \forall \ranking_z\in  \rankings$ satisfying $r_{l_i} r_{l_j} \sqsubset\ranking_z$. Thus, the entry $A^{\ranking_l}[j,i]=1$ represents  $r_{l_i} r_{l_j} \stackrel{q}{\sqsubset}\! \rankings$. Moreover, note that $\mathbf{e}^T \mathbf{A}^{\ranking_l} \mathbf{e}$ gives the sum of the all entries in matrix $\mathbf{A}^{\ranking_l}$. Therefore, the result of (\ref{eq:k1Calcul}) and (\ref{eq:k2Calcul}) can be further obtained based on Definition \ref{def:IndividualConsensusScore}.	
\end{proof}

The matrix $\mathbf{A}^{\ranking_l}$ provides a proper representation of the $q$-support patterns in $\ranking_l$. This representation can further facilitate the analysis of the commonality that individual rankings share with the others. Based on Theorem \ref{theoremAmatrix}, we develop Algorithm \ref{algo:ConsensusScore} to calculate the consensus scores and characterize the $q$-support patterns more efficiently.
\begin{algorithm}[!t]
\caption{Quantifying consensus with matrix representation}
\label{algo:ConsensusScore}
\SetAlgoLined
\KwData{A set of rankings $\rankings$, the value of $q$}
\KwResult{$\kappa_{1}^{\ranking_l}(q),  \kappa_{2}^{\ranking_l}(q), l = 1, 2, \cdots, N$; $\subseqs_1^{\ranking_l}(q)$, $\subseqs_2^{\ranking_l}(q)$; $\bar{\kappa}_{1}(q),  \bar{\kappa}_{2}(q)$; $\subseqs_1(q)$, $\subseqs_2(q)$}

Initialize $\mathbf{A}^{\ranking_l}, l = 1, 2, \cdots, N$ with zero matrices 

\For{$l=1$ to $N$}{
  $m$ $\leftarrow$ Length of $\ranking_l$
  
    \For{$i=1$ to $m$}{
   	   \For{$j=i$ to $m$}{
	   	   $f(r_{l_i}, r_{l_j}) = 0$\\
	       \uIf{$\left(l\!>\!1, N\!-\!l\!+\!1\geq q, \exists x\! \in \![1,l \!-\!1]\right)$ or $\left(l\!>\!1, N\!-\!l\!+\!1<q, \exists x\! \in \![1, N\!-\!q\!+\!1]\right)$ such that    $r_{l_i} r_{l_j}\sqsubset  \ranking_x    $} {
			    $A^{\ranking_l}[j,i]=A^{\ranking_x}[\pi(r_{l_j},\ranking_x), \pi(r_{l_i},\ranking_x)]$
         
		         continue
	       }
	       \ElseIf{$N-l+1\geq q$}
	           {\For{$z=l$ to $N$}{
		           Calculate $\pi(r_{l_i}, \!\ranking_z), \pi(r_{l_j},\ranking_z)$ by Eq. (\ref{eq:PositionFunction})
	           
		           Calculate 
		           	\begin{eqnarray}
					f(r_{l_i}, r_{l_j}) += \begin{cases}
					H(\pi(r_{l_i}, \ranking_z)) , ~~~~~~~~~~~~~~~~~~~~~~~~~~~~~~~~\text{if}~ i =j\\	
					H(\pi(r_{l_j}, \ranking_z) \!-\! \pi(r_{l_i}, \ranking_z)) 
			H(\pi(r_{l_i}, \ranking_z)), ~~ \text{otherwise}\nonumber
					\end{cases}
					\end{eqnarray}
					
					\uIf{$N-z+f(r_{l_i}, r_{l_j})<q$} {
						break				
                   }
       	
				}
			}
			\uIf{$f(r_{l_i}, r_{l_j})\geq q$} {
			$A^{\ranking_l}[j,i] = 1$
			}
		}	
	}	

   Calculate $\kappa_{1}^{\ranking_l}(q), \kappa_{2}^{\ranking_l}(q)$ by  Eqs. \eqref{eq:k1Calcul} and \eqref{eq:k2Calcul}
   
   Get $\subseqs_1^{\ranking_l}(q), \subseqs_2^{\ranking_l}(q)$ based on $\mathbf{A}^{\ranking_l}$

}

Calculate $\bar{\kappa}_1(q), \bar{\kappa}_2(q)$ by Eqs. (\ref{eq:overallK1score}) and (\ref{eq:overallK2score})

Get $\subseqs_1(q), \subseqs_2(q)$ by $\subseqs_1(q) = \displaystyle \mathop{\cup}_{\ranking_l \in \rankings}^{} \subseqs_1^{\ranking_l}(q), ~ \subseqs_2(q) = \displaystyle \mathop{\cup}_{\ranking_l \in \rankings}^{} \subseqs_2^{\ranking_l}(q)$

\textbf{return} \{$\kappa_{1}^{\ranking_l}(q),  \kappa_{2}^{\ranking_l}(q), l = 1, 2, \cdots, N$; $\subseqs_1^{\ranking_l}(q), \subseqs_2^{\ranking_l}(q)$; $\bar{\kappa}_{1}(q),  \bar{\kappa}_{2}(q)$; $\subseqs_1(q), \subseqs_2(q)$\}
\end{algorithm}
Suppose $q \!= \! \lceil \!\frac{2N}{3}\! \rceil$, which means that we consider $r_{l_i} r_{l_j}$ as a common pattern if it is contained by at least two third of the rankings.  For $\ranking_l, l \!> \!\lfloor\!\frac{N}{3}\!\rfloor \!+ \!1$, if $r_{l_i} r_{l_j}$ is not included by one of the first $\lfloor\!\frac{N}{3} \!\rfloor\!+\!1$ rankings of the ranking set,  $A^{\ranking_l} \! [j,i]$ must be zero. If  $r_{l_i} r_{l_j}$ is a $\lceil\! \frac{2N}{3} \!\rceil$-support pattern, it must be included by one of the rankings $\ranking_x \! \in \! \{\ranking_1,\ranking_2,\dots,\ranking_{\lfloor \! \frac{N}{3} \!\rfloor \!+\!1}\}$. Thus, we do not need to calculate $A^{\ranking_l}[j,i]$ by always checking all the rankings. Line 7 in Algorithm \ref{algo:ConsensusScore} checks if $r_{l_i} r_{l_j}$ of $\ranking_l$ is included by a ranking $\ranking_x $ for which matrix $\mathbf{A}^{\ranking_x}$ has already been constructed. If the number of the rankings whose corresponding matrix is not constructed is greater than $q$, we look for $\ranking_x$  in the considered rankings   $\{\ranking_1,\ranking_2,\dots,\ranking_{l\!-\!1} \}$;  otherwise we only check if there is an $\ranking_x$ in the first $N\!-\!q\!+\!1$ rankings.  As shown in Lines 8 and 9, if $r_{l_i} r_{l_j}$ has been considered in a constructed matrix for $\ranking_x$, it is not necessary to recalculate the corresponding entry of the current matrix $\mathbf{A}^{\ranking_l}$  and the entry is equal to that of $\mathbf{A}^{\ranking_x}$ corresponding to the pattern. Otherwise, as in Line 10,  only when  the number of the rankings $\{\ranking_l,\ranking_{l+1},\dots,\ranking_N \}$ is no less than $q$,   $r_{l_i} r_{l_j}$ has the possibility to be a $q$-support pattern. In this way, the computation cost can be significantly reduced.  From Lines 11 to 16, $f(\!r_{l_i}, r_{l_j}\!)$ accumulates the number of rankings containing $r_{l_i} r_{l_j}$. To further improve the computation efficiency,  the  sum of $f(\!r_{l_i}, r_{l_j}\!)$ and the number of the remaining rankings is checked during the accumulation process. If it is less than $q$,  then $r_{l_i} r_{l_j}$ has no chance to be a $q$-support pattern and there is no need to check if the remaining rankings contain $r_{l_i} r_{l_j}$.

The following example shows how the matrix representation can be used to evaluate the ranking consensus. 
\begin{example}
\label{example}
Consider a set of rankings $\rankings\!=\!\{ \ranking_1\! =\! (a, \!b, \!c, \!d,  \!e, \!f), \ranking_2 \!=\! ( b,  \! c,  \!d,  \!e,   \!f,  \!a  ),
\ranking_3 \!=\! (b,  \!d,  \!a,  \!g,  \!h,  \!f), \ranking_4 \!=\! (b,a,c,d,f,e)\}$ over candidates $\{a,  \!b,  \! c,  \!d,  \!e, \! f,  \!g,  \!h\}$, and let $q=3$. We have 
\begin{equation*}	
	\mathbf{A}^{\ranking_1}
	\!=\! \bordermatrix{~ \!&\! a \!&\! b \!&\! c \!&\! d \!&\!e \!&\!f\!\cr
		a \!&\! 1\!&\! 0 \!&\! 0 \!&\! 0 \!&\! 0 \!&\! 0\! \cr
		b \!&\! 0    \!&\! 1 \!&\! 0 \!&\! 0 \!&\! 0 \!&\! 0 \!\cr
		c \!&\! 0 \!&\! 1 \!&\! 1 \!&\! 0 \!&\! 0 \!&\! 0 \!\cr
		d \!&\! 0 \!&\! 1 \!&\! 1 \!&\! 1 \!&\! 0 \!&\! 0\! \cr
		e \!&\! 0 \!&\! 1 \!&\! 1 \!&\! 1 \!&\! 1 \!&\! 0 \!\cr
		f \!&\! 1 \!&\! 1 \!&\! 1 \!&\! 1 \!&\! 0 \!&\! 1\! \cr}, ~
		\mathbf{A}^{\ranking_2}
	\!=\! \bordermatrix{~ \!&\! b \!&\!c \!&\!d \!&\!e \!&\!f \!&\!a\!\cr
		b \!&\! 1 \!&\! 0 \!&\! 0 \!&\! 0 \!&\! 0 \!&\! 0 \!\cr
		c \!&\! 1 \!&\! 1 \!&\! 0 \!&\! 0 \!&\! 0 \!&\! 0\! \cr
		d \!&\! 1 \!&\! 1 \!&\! 1 \!&\! 0 \!&\! 0 \!&\! 0 \!\cr
		e \!&\! 1 \!&\! 1 \!&\! 1 \!&\! 1 \!&\! 0 \!&\! 0\! \cr
		f \!&\! 1 \!&\! 1 \!&\! 1\!&\! 0 \!&\! 1 \!&\! 0\! \cr
		a \!&\! 1 \!&\! 0 \!&\! 0 \!&\! 0 \!&\! 0 \!&\! 1\! \cr} 
	\end{equation*}	
	\begin{equation*}	
	\mathbf{A}^{\ranking_3}
	\!=\! \bordermatrix{~ \!&\! b \!&\! d \!&\! a \!&\! g \!&\! h \!&\! f \!\cr
		b \!&\! 1 \!&\! 0 \!&\! 0 \!&\! 0 \!&\! 0 \!&\! 0  \!\cr
		d\!&\! 1 \!&\! 1 \!&\! 0 \!&\! 0 \!&\! 0 \!&\! 0 \!\cr
		a \!&\! 1 \!&\! 0 \!&\! 1 \!&\! 0 \!&\! 0 \!&\! 0\!\cr
		g \!&\! 0 \!&\! 0 \!&\! 0 \!&\! 0 \!&\! 0 \!&\! 0 \!\cr
		h \!&\! 0 \!&\! 0 \!&\! 0 \!&\! 0 \!&\! 0 \!&\! 0 \!\cr
		f \!&\! 1 \!&\! 1 \!&\! 1 \!&\! 0 \!&\! 0 \!&\! 1\! \cr}, ~
 		\mathbf{A}^{\ranking_4}
	\!=\! \bordermatrix{~ \!&\! b \!&\!a \!&\!c \!&\!d \!&\!f \!&\!e\!\cr
		b \!&\! 1 \!&\! 0 \!&\! 0 \!&\! 0 \!&\! 0 \!&\! 0 \!\cr
		a \!&\! 1 \!&\! 1 \!&\! 0 \!&\! 0 \!&\! 0 \!&\! 0\! \cr
		c \!&\! 1 \!&\! 0 \!&\! 1 \!&\! 0 \!&\! 0 \!&\! 0 \!\cr
		d \!&\! 1 \!&\! 0 \!&\! 1 \!&\! 1 \!&\! 0 \!&\! 0\! \cr
		f \!&\! 1 \!&\! 1 \!&\! 1 \!&\! 1 \!&\! 1 \!&\! 0 \!\cr
		e \!&\! 1 \!&\! 0 \!&\! 1 \!&\! 1 \!&\! 0 \!&\! 1 \!\cr}.
	\end{equation*}	
By Eq. (\ref{eq:k1Calcul}) and Eq. (\ref{eq:k2Calcul}), the following result can be obtained
\begin{center}
			\begin{tabular}{|c|c|c|c|c|}
				\hline
				 $l$ & $1$ & $2$ & $3$ & $4$ \\
				  \hline
				  $\kappa_{1}^{\ranking_l}(3) $  & 1.00 & 1.00 & 0.67 & 1.00 \\
				  \hline
				  $\kappa_{2}^{\ranking_l}(3) $  & 0.67 & 0.67 & 0.33 & 0.73  \\
				  \hline
			\end{tabular}
\end{center}

\noindent The overall consensus scores are
\begin{equation*}	
\bar{\kappa}_1(3) = 0.92,~
\\\bar{\kappa}_2(3) = 0.60.
\end{equation*}	
Since $\mathbf{A}^{\ranking_l}[j,i]$ represents	if $r_{l_i} r_{l_j}$ is a $q$-support pattern, it can be known
$\subseqs_1^{\ranking_1}(3) \! = \! \{a, b, c, d, e, f\}, ~ \subseqs_2^{\ranking_1}(3) \! = \! \{af, bc, bd, be, bf, cd, \\ ce, cf, de, df \} $, $\subseqs_1^{\ranking_2}(3) \!=\! \{b, c,  d, e, f, a\}, \subseqs_2^{\ranking_2}(3) \!=\! \{ bc, bd, be, bf, ba, cd, ce, cf, de, df\}, ~ \subseqs_1^{\ranking_3}(3) \!= \!\{b, d,  a,  f\}, ~\subseqs_2^{\ranking_3}(3) \!= \!\{ bd, ba, \\ bf, df, af \}, ~\subseqs_1^{\ranking_4}(3)  = \{b, a, c, d, f, e\}, ~ \subseqs_2^{\ranking_4}(3)  = \{ ba, bc, bd, bf, be, af, cd, cf, ce, df, de\}$. Furthermore, the sets of  $q$-support patterns of the whole ranking set are $\subseqs_1(3)  \!=\!  \subseqs_1^{\ranking_1}(3) \cup \subseqs_1^{\ranking_2}(3) \cup \subseqs_1^{\ranking_3}(3) \cup \subseqs_1^{\ranking_4}(3)  \!=\! \{a, b, c, d, e, f\}, ~ \subseqs_2(3)  \!=\!  \subseqs_2^{\ranking_1}(3) \cup \subseqs_2^{\ranking_2}(3) \cup \subseqs_3^{\ranking_3}(3) \cup \subseqs_4^{\ranking_4}(3)  \!=\!  \{ af, ba, bc,bd, be, bf, cd, ce, cf, de, df\}$. 		
\end{example}


\section{Quantifying consensus with consideration of positions and position gaps}
\label{section: Weighted Consensus Quantifying}
The rank positions of an item and the position gaps of pairwise items may be significantly different in a ranking set. Consider the items $a$ and $f$ in Example \ref{example}. The rank positions of item $a$ are
$\pi(a, \ranking_1) \! = \!  1, \pi(a, \ranking_2)  \! = \!  6, \pi(a, \ranking_3)  \! = \!  3, \pi(a, \ranking_4)  \! = \!  2$ and the position gaps of the two items are $\pi(f, \ranking_1)  \! - \!  \pi(a, \ranking_1)  \! = \!  5, \pi(f, \ranking_3) \!  - \!  \pi(a, \ranking_3) \!  = \!  3, \pi(f, \ranking_4) \!  -  \! \pi(a, \ranking_4) \!  = \!  3$. These differences influence the ranking consensus. However, the consensus scores defined in the previous section only involve the existence of $q$-support patterns. To reflect the importance of these position and gap information, 
the following definition presents an extension to Eqs. \eqref{eq:individualK1} and \eqref{eq:individualK2},  for quantifying the consensus of a ranking set more effectively.  
\begin{definition}[Weighted individual consensus scores]
	The weighted consensus scores of ranking $\ranking_l\in\rankings$ are
	\begin{eqnarray}
		\!\!\!\!\!\!\kappa_1^{\ranking_l}(q)\!\!\!\!\!&&= \frac{1}{N_1^{\ranking_l}} \sum_{r_{l_i} \in \subseqs_1^{\ranking_l}(q)} \gamma^{h(r_{l_i}, \ranking_l)} \label{eq:weight:individual:k1}\\
		\!\!\!\!\!\!\kappa_2^{\ranking_l}(q)\!\!\!\!\!&&= \frac{1}{N_2^{\ranking_l}} \sum_{r_{l_i} r_{l_j}\in \subseqs_2^{\ranking_l}(q)} \lambda^{d(r_{l_i}, r_{l_j}, \ranking_l)}, \label{eq:weight:individual:k2}
		\end{eqnarray}	
where the constants $0\!<\!\gamma \! \leq \! 1$ and $0\! <\! \lambda \!\leq \! 1$ are the weights, $h(r_{l_i}, \ranking_l)$ is the deviation of the position of $r_{l_i}$ in $\ranking_l$ from its average position in the ranking set, and $d(r_{l_i}, r_{l_j}, \ranking_l)$ is the deviation of the position gaps between  $r_{l_i}$ and $r_{l_j}$ in $\ranking_l$ from the average.  
\end{definition}

The deviations  $h(r_{l_i}, \ranking_l)$  and $d(r_{l_i}, r_{l_j}, \ranking_l)$ are calculated as follows. 
For ranking $\ranking_l \in \rankings$, we have the sets $\subseqs_1^{\ranking_l}(q)$ and $\subseqs_2^{\ranking_l}(q)$ of the $q$-support patterns defined as Eqs. (\ref{eq:setSingleqSupportItems}) and (\ref{eq:setqSupportPatterns}), the function $f(r_{l_i}, r_{l_j})$ in the form of Eq. (\ref{eq:fFunctionCount}), and the subset $\rankings'(r_{l_i}, r_{l_j})$ of $\rankings$ containing pattern $r_{l_i} r_{l_j}$ as Eq. (\ref{eq:RankingSetIncludeItems}). The average position of item $r_{l_i}$ in the ranking set is defined as
\begin{equation}
\bar{\pi}(r_{l_i}) = \frac{1}{f(r_{l_i}, r_{l_i})} \sum_{\ranking_z \in \rankings'(r_{l_i}, r_{l_i})} \pi(r_{l_i}, \ranking_z).
\end{equation}
The deviation $ h(r_{l_i}, \ranking_l)$ is
\begin{equation}
\label{eq:postionVar}
h(r_{l_i}, \ranking_l) =|\pi(r_{l_i}, \ranking_l) - \bar{\pi}(r_{l_i})|.
\end{equation}
The position gap between $r_{l_i}$ and $r_{l_j}$ in ranking $\ranking_z$ is
\begin{equation}
\omega(r_{l_i}, r_{l_j}, \ranking_z)= \pi(r_{l_j}, \ranking_z) - \pi(r_{l_i}, \ranking_z).
\end{equation} 
The average position gap of $r_{l_i}$ and $r_{l_j}$ in the ranking set is defined as
\begin{equation}
\bar{\omega}(r_{l_i}, r_{l_j}) = \frac{1}{f(r_{l_i}, r_{l_j})} \sum_{\ranking_z \in \rankings'(r_{l_i},r_{l_j})} \omega(r_{l_i}, r_{l_j}, \ranking_z).
\end{equation}
The deviation $d(r_{l_i}, r_{l_j}, \ranking_l)$ is
\begin{equation}
\label{eq:gapVar}
d(r_{l_i}, r_{l_j}, \ranking_l) =  |\omega(r_{l_i}, r_{l_j}, \ranking_l)-\bar{\omega}(r_{l_i}, r_{l_j})|. \nonumber
\end{equation}
From the definition, it can be known that smaller values of $\gamma$ and $\lambda$ reflect greater impact of the deviations of item  positions and position gaps in rankings on the consensus scores. It is worth noting that the consensus scores defined in the previous section are a special case of the weighted consensus scores with $\gamma = 1, \lambda =1$. Here, we  do not need to make any change to  the overall consensus scores defined in Definition \ref{def:OverallConsensusScore}.

To calculate the weighted consensus scores with the matrix representation, Eq. (\ref{eq:a_ij}) in Theorem 1 is changed to 
\begin{equation} 
\label{eq:AwithWeighting}
	A^{\ranking_l}[j,i] \!= \!
	\begin{cases}
	\gamma^{h(r_{l_i}, \ranking_l)},  & \! \!\!\text{if} ~i= j~  \text{and}~ f(r_{l_i}, r_{l_j}) \! \geq\! q\\
	 \lambda^{d(r_{l_i}, r_{l_j}, \ranking_l)},  &\!\!\! \text{if} ~i< j~  \text{and}~ f(r_{l_i}, r_{l_j}) \!\geq\! q\\
	\!0,  & \text{otherwise}.
	\end{cases}
	\end{equation}
Small change will be needed in Algorithm \ref{algo:ConsensusScore}. We follow the steps of Algorithm \ref{algo:ConsensusScore} and change the way to calculate $A^{\ranking_l}[j,i]$ in Line 8 to the following form
\begin{equation*} 
 A^{\ranking_l}[j,i]=
 	\begin{cases}
     H\left(A^{\ranking_x}[\pi(r_{l_j},\ranking_x), \pi(r_{l_i},\ranking_x)]\right) \gamma^{h(r_{l_i}, \ranking_l)}, & \text{if} ~i = j\\
      H\left(A^{\ranking_x}[\pi(r_{l_j},\ranking_x), \pi(r_{l_i},\ranking_x)]\right) \lambda^{d(r_{l_i}, r_{l_j}, \ranking_l)}, & \text{if} ~i < j.
     \end{cases}
\end{equation*}
Line 19 is replaced by
\begin{equation*}
	A^{\ranking_l}[j,i] \!= \!
	\begin{cases}
	\gamma^{h(r_{l_i}, \ranking_l)},  & \! \!\!\text{if} ~i= j \\
	 \lambda^{d(r_{l_i}, r_{l_j}, \ranking_l)},  &\!\!\! \text{if} ~i< j,
	 \end{cases}
\end{equation*}
and meanwhile the average position $\bar{\pi}(r_{l_i})$ or the average position gap $\bar{\omega}(r_{l_i}, r_{l_j})$ is recorded in here for further use in Line 8.

\begin{remark}[Rankings with ties]
Rankings with ties are used in the case that the preferences over some items are identical.  Let $\mathbf{r}_z = (\mathcal{T}_{z_1},\mathcal{T}_{z_2}, \cdots, \mathcal{T}_{z_n})$ be a ranking with ties, where $\mathcal{T}_{z_i}, i \in [1,n]$ is a set of items with identical preference. For $i<j$, every item in $\mathcal{T}_{z_i}$ is more preferred than all the items in $\mathcal{T}_{z_j}$. The proposed approach can be extended to rankings with ties by making small change to the position function. Specifically, we can replace Eq. (\ref{eq:PositionFunction}) with
\begin{equation*}
\pi\left(r_{l_i}, \ranking_z\right) =
\begin{cases}
p, & \text{if}~ r_{l_i} \in \mathcal{T}_{z_p} \\
0, & \text{otherwise}
\end{cases}	
\end{equation*}
to make the approach applicable to evaluate the consensus of rankings with ties. 
\end{remark}

%
%

\section{Detecting outliers}
\label{section: Detecting outliers}

The individual consensus scores $\kappa_{1}^{\ranking_l}(q)$ and $\kappa_{2}^{\ranking_l}(q)$ directly reflect the (weighted) numbers of $q$-support patterns that $\ranking_l$ shares with the other rankings in $\rankings$. For instance, ranking $\ranking_3$ in Example \ref{example} shares less $3$-support patterns with the others, thus it has much lower consensus scores. This can be used to detect outlier rankings, which have low consensus with most rankings. The following outlier detection method is naturally developed from the consensus quantifying approach.

Consider a ranking set $\rankings$ with overall consensus scores $\bar{\kappa}_{1}(q)$ and $\bar{\kappa}_{2}(q)$ for a given $q$. Define the relative deviations of the individual consensus scores of ranking $\ranking_l \in \rankings$ from the overall consensus scores as
\begin{eqnarray}	
v_{1}^{\ranking_l}(q) &=& \frac{\kappa_{1}^{\ranking_l}(q) - \bar{\kappa}_{1}(q)}{\bar{\kappa}_{1}(q)}	\\
v_{2}^{\ranking_l}(q) &=& \frac{\kappa_{2}^{\ranking_l}(q) - \bar{\kappa}_{2}(q)}{\bar{\kappa}_{2}(q)}.
\end{eqnarray}
Note that  $v_{1}^{\ranking_l}(q) \!<\! 0$ and $v_{2}^{\ranking_l}(q)\!<\!0$ imply that the ranking $\ranking_l$ has lower consensus scores than the overall averages. For given constants $\epsilon_1\!>\!0$ and $\epsilon_2\!>\!0$, if $v_{1}^{\ranking_l}(q)\! < \! - \epsilon_1$ or $v_{1}^{\ranking_2}(q)\! <\! -\epsilon_2$, we regards $\ranking_l$ as an outlier of the ranking set. The values of $\epsilon_1, ~\epsilon_2$ depend on the specific need for a system.

This outlier detection method can be used to figure out irrelevant rankings in the ranking set and consequently identify the majority of rankings with higher consensus.  It is of great importance in many scenarios, e. g., design of auto-suggestion queries in search engine. It is worth noting that one potential application of the obtained detection method is to improve rank aggregation. Rank aggregation is the task of aggregating the preferences of different agents to generate a final ranking. 
The outliers of rankings/agents play a negative role in drawing a consensus ranking. Even though many existing studies have been carried out on rank aggregation \cite{volkovs2014new,chen2015spectral,caragiannis2018optimizing}, there is still room to improve aggregated rankings so that the aggregated result is as close to the ground truth as possible.  This will be studied in a separate paper. 

\section{Experimental studies}
\label{section: Experiment}

This section shows how the proposed approaches can be used to evaluate consensus for a set of rankings. The source code is available at   \url{https://github.com/zhiweiuu/secs}.

\subsection{Analysis of the Mechanical Turk Dots datasets}
The  Mechanical Turk Dots datasets \cite{mao2013better} include four publicly available datasets obtained for four dots tasks. These datasets each contain rankings obtained by 794 to 800 voters over four candidates. Each candidate corresponds to a certain number of random dots. The voters are asked to rank the candidates from those with the least dots to the most. Each task contains candidates with $200$, $200 \! + \! i$, $200 \! + \! 2i$, and $200+ \!  3i$ dots, where $i \!  = \!  {3, 5, 7, 9}$ respectively for the four tasks. Figure \ref{fig:SpearmanTurkDots} shows the proportions of rankings in each dataset with different Spearman's $\rho$ to the ground truth ranking.  The values of different Spearman's $\rho$ are distinguished by colors. It can be seen that the proportions of rankings with high Spearman coefficients $0.8$ and $1.0$ increase  from  Dataset 1 to Dataset 4, while that with coefficient $0.4$ decreases significantly. The ranking consensus degrees seem increasing from Dataset 1 to Dataset 4. We apply the proposed approach to accurately compare these datasets. 
\begin{figure}[t]
  \centering
  \includegraphics[width=2.9in]{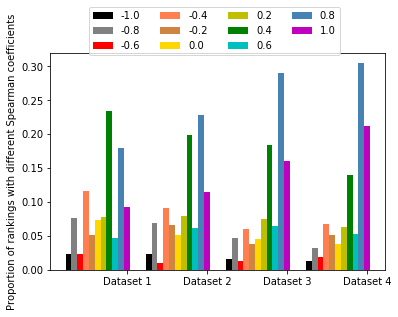}
	\caption{Spearman’s $\rho$ between the rankings and the ground truth ranking}
    \label{fig:SpearmanTurkDots}
\end{figure}

The overall consensus scores without weighting are first considered. Since the datasets have complete rankings, i.e., all the candidates under consideration are ranked in the rankings, the consensus scores of the single items satisfy $\bar{\kappa}_1(q)\!= \!4 $ for all $q$ and all the datasets. Figure \ref{fig:TurkDotsK2meanNoWeight} gives the overall consensus scores  $\bar{\kappa}_2(q)$ with respect to $\frac{q}{N}$, where  $\frac{q}{N}\geq 0.5$ indicating that  the commonality embedded in half or more than half of the rankings is evaluated. The trend of the overall consensus scores for the four datasets is clear. Dataset 4 has the largest overall consensus score, which indicates that Dataset 4 has the most $q$-support common patterns. Specifically, it can be seen from the figure that, when $\frac{q}{N}$ is 0.5, the consensus score $\bar{\kappa}_2(q)$ is $0.59, 0.62, 0.68$, and $0.71$ respectively for Dataset 1, 2, 3 and 4. This means that on average, $59.00\%, 62.00\%, 68.00\%$, and $71.00\%$  of the pairwise patterns of  a ranking are $\lceil\frac{N}{2}\rceil$-support patterns  in Dataset 1, 2, 3 and 4, respectively.  As the value of $q$ increases, the consensus scores decrease. When  $\frac{q}{N}$ reaches 0.67, the consensus score is zero for Dataset 1, which means that arbitrary $q\geq 0.67N$ rankings in the dataset have no common pattern. On the other hand,  the consensus scores are 0.12, 0.37, 0.38  for Dataset 2, 3, 4. In other words, on average, $12.00\%, 37.00\%, 38.00\%$ of the patterns of a ranking are supported by at least $0.67N$ rankings in the corresponding dataset. 
\begin{figure}[!t]
    \centering
    \includegraphics[width=3.4in]{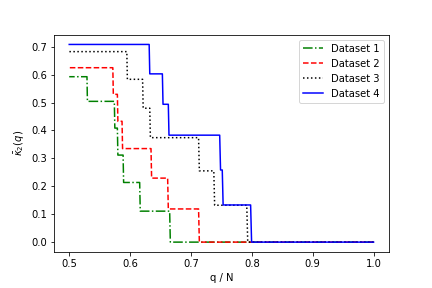}
    \caption{Consensus scores $\bar{\kappa}_2(q)$  of Dots datasets  without weighting}
    \label{fig:TurkDotsK2meanNoWeight}
\end{figure}

%
The overall consensus scores with weightings are then evaluated. Figure   \ref{fig:Weighted consensus scores of of Dots datasets} shows the consensus scores with respect to the weights $\gamma$ and $\lambda$ for a fixed $q\!=\!
\lceil \frac{N}{2} \rceil$. As shown, $\bar{\kappa}_{1}(\lceil\frac{N}{2}\rceil)$ and $\bar{\kappa}_{2}(\lceil\frac{N}{2}\rceil)$ decrease with the increase of weightings on the deviations of positions and position gaps.  Dataset 1 has the lowest overall consensus scores and Dataset 4 has the highest.  The ratios of the consensus scores between Dataset 4 and Dataset 3, Dataset 3 and Dataset 2, and Dataset 2 and Dataset 1 are shown in Table \ref{table:RatiosConsensusBetweenDatasets} for the cases without weighting and with weighting parameters $\gamma = 0.5, \lambda=0.5$. By comparing the two cases, it can be found that the ratios with weightings on the deviations of the position and position gaps are higher than those without weightings. This reveals that the differences of  the positions of the single $q$-support items and the position gaps of  the $q$-support patterns decrease from Dataset 1 to Dataset 4.

\begin{table}[!h] 
\centering
\caption{ Ratios of the consensus scores between datasets} 
\label{table:RatiosConsensusBetweenDatasets}
			\begin{tabular}{|c|c|c|c|}
				\hline
				  & Dataset 4/Dataset 3 &  Dataset 3/Dataset 2 & Dataset 2/Dataset 1 \\
				  \hline
				  $\bar{\kappa}_{1}(\lceil\frac{N}{2}\rceil), \gamma = 1$	  & 1.00 & 1.00 & 1.00 \\
				  \hline
				  $\bar{\kappa}_{1}(\lceil\frac{N}{2}\rceil), \gamma = 0.5$	  & 1.02 & 1.04 & 1.04\\
				  \hline
				  $\bar{\kappa}_{2}(\lceil\frac{N}{2}\rceil), \lambda=1$  & 1.04 & 1.09 &  1.05  \\
				  \hline
				  $\bar{\kappa}_{2}(\lceil\frac{N}{2}\rceil), \lambda=0.5$  & 1.05 & 1.11 & 1.07  \\
				  \hline
			\end{tabular}
\end{table}

\begin{figure}[!t]
    \centering
    \subfigure[$\bar{\kappa}_{1}(\lceil\frac{N}{2}\rceil)$ with respect to $\gamma$]
    {
        \includegraphics[width=3.2in]{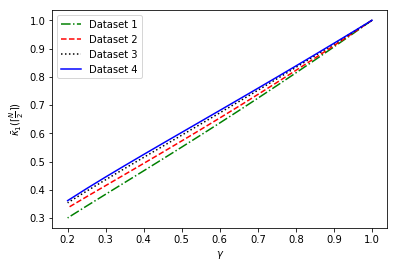}
        \label{fig:TurkDotsK1meanWeight}
    }
            \subfigure[$\bar{\kappa}_{2}(\lceil\frac{N}{2}\rceil)$ with respect to $\lambda$]
    {
        \includegraphics[width=3.2in]{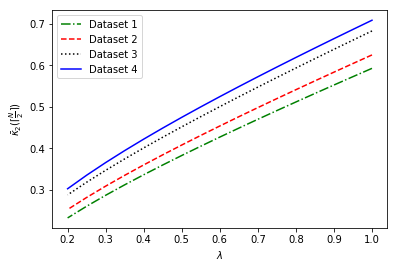}
        \label{fig:TurkDotsK2meanWeight}
    }
    \caption{Weighted consensus scores of Dots datasets}
    \label{fig:Weighted consensus scores of of Dots datasets}
\end{figure}

The relative deviations of $\kappa_2^{\ranking_l}(\lceil\frac{N}{2}\rceil)$ from the overall consensus score $\bar{\kappa}_{2}(\lceil\frac{N}{2}\rceil)$ is also studied to  verify the effectiveness of the proposed outlier detection method.  By choosing $\lambda = 0.5$, the result in  Table \ref{table:TurkConsensusDeviation} can be obtained. The deviations are very high for $\ranking_{21}, \ranking_{24}, \ranking_{22}, \ranking_{13}$
of Dataset 1, $\ranking_{20}, \ranking_{22}, \ranking_{19}, \ranking_{17}$ of Dataset 2, $\ranking_{19}, \ranking_{20}, \ranking_{22}, \ranking_{15}$ 
of Dataset 3, and $\ranking_{21}, \ranking_{22}, \ranking_{24},  \ranking_{20}$ of Dataset 4. These rankings are regarded as outliers of the datasets. They are $ (4, 3, 2, 1)$, $(4, 3, 1, 2)$, $(4, 2, 3, 1)$, $(3, 4, 2, 1)$ respectively in each dataset. Note that the Spearman’s $\rho$ between $(4, 3, 2, 1)$ and the ground truth $(1, 2, 3, 4)$ are $-1$, and all the Spearman coefficients  of the rest three to the ground truth are $-0.8$. After deleting these outlier rankings, the consensus score $\bar{\kappa}_{1}(\lceil\frac{N}{2}\rceil)$  with $\gamma = 0.5$ increases from $0.55, 0.57, 0.59, 0.60$ to $0.58, 0.59, 0.61, 0.62$ for Dataset 1, 2, 3, 4, respectively. The consensus score $\bar{\kappa}_{2}(\lceil\frac{N}{2}\rceil)$ changes from $0.38, 0.41, 0.45, 0.47$ to $0.42, 0.44, 0.48, 0.49$ for the four datasets. This confirms the effectiveness of the proposed outlier detection method. 

\begin{table}[!h] 
\centering
\caption{ Deviations of the consensus scores} 
\label{table:TurkConsensusDeviation}
			\begin{tabular}{|c|c|c|c|c|}
				\hline
				$q = \lceil\frac{N}{2}\rceil$ & Dataset 1 & Dataset 2 & Dataset 3 & Dataset 4 \\
				  \hline
				  $v_{2}^{\ranking_1}(q)$  & 0.72  & 0.68   &  0.56    & 0.55 \\
				  $v_{2}^{\ranking_2}(q)$  &  0.14  & 0.38  & 0.29 &  0.20 \\
				  $v_{2}^{\ranking_3}(q)$  & 0.44    &  0.29   & 0.28   &  0.18 \\
				  $v_{2}^{\ranking_4}(q)$  & 0.44 & 0.36  & 0.15  &   0.08 \\
				  $v_{2}^{\ranking_5}(q)$  &  0.09   &  0.08   &   -0.03    &   -0.09  \\
				  $v_{2}^{\ranking_6}(q)$  & 0.06 &  0.06 & -0.03  &   -0.21 \\
				  $v_{2}^{\ranking_7}(q)$  & 0.47  & -0.01   &  -0.15   &   -0.10 \\
				  $v_{2}^{\ranking_8}(q)$  &   0.11   &     -0.27  & -0.04 &  -0.11\\
				  $v_{2}^{\ranking_9}(q)$  &  0.02 &  -0.04   &  -0.35     &   -0.41  \\
				  $v_{2}^{\ranking_{10}}(q)$  &   -0.22 &  0.02  &  -0.37 &  -0.21 \\
				  $v_{2}^{\ranking_{11}}(q)$  &   -0.29 &   -0.28    &   -0.17  &     -0.20  \\
				  $v_{2}^{\ranking_{12}}(q)$  &  -0.38 & -0.11  &  -0.42   &   -0.41 \\
				  $v_{2}^{\ranking_{13}}(q)$  & \textbf{-0.71} & -0.13 &  -0.20  &   -0.56    \\
				  $v_{2}^{\ranking_{14}}(q)$  &  -0.25 &  -0.15 & -0.49  & -0.35   \\
				  $v_{2}^{\ranking_{15}}(q)$  & -0.06  & -0.46   &  \textbf{-0.74 }  &  -0.21 \\
				  $v_{2}^{\ranking_{16}}(q)$  & -0.09 &  -0.45 &   -0.20  &  -0.54    \\
				  $v_{2}^{\ranking_{17}}(q)$  & -0.36 & \textbf{-0.72}   &     -0.28  & -0.49  \\
				  $v_{2}^{\ranking_{18}}(q)$  &  -0.39 & -0.37 & -0.50  &   -0.47 \\
				  $v_{2}^{\ranking_{19}}(q)$  &   -0.40   & \textbf{ -0.74 }  &   \textbf{-1.00}  & -0.55  \\
				  $v_{2}^{\ranking_{20}}(q)$  &  -0.46 &  \textbf{-1.00}   & \textbf{-0.75}  &   \textbf{-0.74}   \\
				  $v_{2}^{\ranking_{21}}(q)$  & \textbf{-1.00} & -0.44 &  -0.50  &   \textbf{-1.00}   \\
				  $v_{2}^{\ranking_{22}}(q)$  & \textbf{-0.73} & \textbf{-0.75}  &   \textbf{-0.77}  &   \textbf{ -0.75}   \\
				  $v_{2}^{\ranking_{23}}(q)$  &   -0.12  & -0.45   &   -0.52  &     -0.54  \\
				  $v_{2}^{\ranking_{24}}(q)$  & \textbf{-0.74} & -0.47  &   -0.52 &   \textbf{-0.77} \\
				  \hline
			\end{tabular}
\end{table}

It is further found that the four datasets have the same set of the  $\lceil \frac{N}{2} \rceil$-support patterns $\subseqs_2(\lceil \frac{N}{2} \rceil) \! =\! \{12, 13, 14, 23, 24, 34\}$.  By aggregating these $\lceil \frac{N}{2} \rceil$-support patterns, we can obtain the ranking $(1, 2, 3, 4)$, i.e., the ground truth ranking. 
This enhances the advantage of the proposed consensus quantifying approach over the pairwise comparison approaches, where no common patterns of the rankings are specified.

\subsection{Evaluation of the information retrieval results of the 2015
CLEFeHealth Lab Task 2}
This experiment focuses on top-$k$ rankings using the dataset of the  CLEF 2015 eHealth Evaluation Lab Task 2 \cite{palotti2015clef}, instead of the complete rankings in the previous section. 
The CLEF 2015 eHealth Evaluation Lab Task 2 aimed to foster the design of web search engines in providing  access to medical information especially for self-diagnosis information, since commercial search engines were far from being effective in the field. 
The problem considered in the task was to retrieve web pages for queries related to different medical conditions. The queries were pre-generated by showing images and videos of medical conditions to potential users. 
There were 67 queries selected to be used in the task for 23 medical conditions, among which 22 conditions had three queries and one condition had one query. The queries were first created in English and then translated into several other languages.
The document collection made available to the participates for information retrieval contains approximately one million web pages on a broad range of health topics. The participates were asked to submit up to ten runs for the English queries. The first run of each team was with the highest priority for selection of documents to contribute to the final assessment. Twelve participating teams submitted their English information retrieval results.

This section evaluates the information retrieval results of the first English runs. Given that the first two pages of a user’s search result probably draw the most attention in practice, the top-20 retrieved documents for each query are considered in the evaluation. 
The conventional Speaman’s $\rho$ and Kendall’s $\tau$ measure the correlation of two complete rankings, as they compare the positions of same items in the two rankings. For this dataset with incomplete rankings, the Sperman’s $\rho$ and Kendall’s $\tau$ for top-$k$ rankings proposed in \cite{fagin2003comparing} are employed to measure the correlations for the 67 queries. Because a typo exists in the $62^{nd} $ query, there is no record of some teams for this query in the dataset. This query is not considered in the following analysis. Since there is no ground truth ranking available, we pairwisely compare the ranking for a query obtained by each team with the rankings of the other teams and take the average. The obtained  comparison results of the team for a query are further aggregated by taking their average.  Figure \ref{fig:TausefdefCLEF}  gives the results of Kendall’s $\tau$. Note that a key parameter $p$ is introduced in the calculation of Kendall’s $\tau$ for top-$k$ rankings in \cite{fagin2003comparing}. This parameter corresponds to the penalty for the case that two items $\sigma_i$ and $\sigma_j$ appears in one ranking $\ranking_l$ and none of them are considered in the other compared ranking $\ranking_z$. In this case, the term $\text{sgn}(\pi(\sigma_i, \ranking_l) - \pi(\sigma_j,  \ranking_l)) \text{sgn}(\pi(\sigma_i, \ranking_z) - \pi(\sigma_j,  \ranking_z))$ is set to be $p$.   We normalized the Kendall’s $\tau$ to the domain of $[-1, 1]$.
The parameter $p = 1$ gives an optimistic approach. It implies that 
$\sigma_i$ and $\sigma_j$ in $\ranking_z$ are regarded as in the same order as in $\ranking_l$ when there is no enough information about them. When $p = 0$, it gives a neutral approach. It can be found in Figure \ref{fig:TausefdefOptiCLEF} and Figure \ref{fig:TausefdefNeutralCLEF} that the Kendall’s coefficients are highly depends on the value of $p$. The result of Spearman’s $\rho$ is shown in Figure \ref{fig:SpearmansefdefCLEF}.  If an item $\sigma_i$ in one top-$k$ ranking $\ranking_l$ does not appear in the other compared top-$k$ ranking $\ranking_z$, then the position $\pi(\sigma_i, \ranking_z)$ is set to $\ell$. In Figure \ref{fig:SpearmansefdefCLEF}, $\ell$ is chosen to be $k + 1$.  The Spearman’s $\rho$ also depends on the value of $\ell$.

\begin{figure}[!t]
    \centering
    \subfigure[Optimistic approach]
    {
        \includegraphics[width=3.2in]{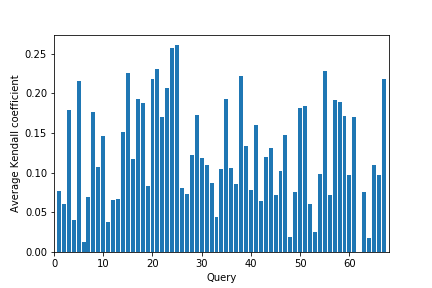}
        \label{fig:TausefdefOptiCLEF}
    }
    \subfigure[Neutral approach]
    {
        \includegraphics[width=3.2in]{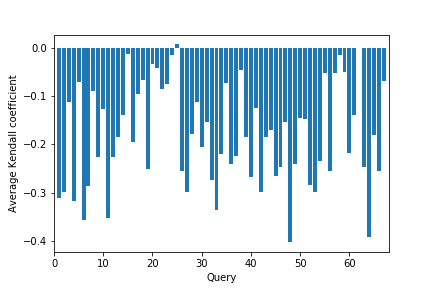}
        \label{fig:TausefdefNeutralCLEF}
    }
    \caption{Average Kendall’s $\tau$ for the   
    ranking sets of the 66 queries obtained by the 12 teams}
    \label{fig:TausefdefCLEF}
\end{figure}

\begin{figure}[!t]
    \centering
    \includegraphics[width=3.2in]{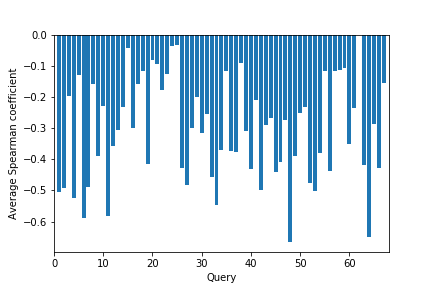}
    \caption{Average Spearman’s $\rho$  for the   
    ranking sets of the 66 queries obtained by the 12 teams}
    \label{fig:SpearmansefdefCLEF}
\end{figure}

Unlike the Speaman’s $\rho$ and Kendall’s $\tau$ for top-$k$ rankings, where assumptions about unknown factors are made without sufficient information and may consequently lead to bias in the measurement results, the proposed approach has no such problem and the consensus of a ranking set is measured more intuitively based on  $q$-support patterns. It provides a clear understanding  about the commonality emmbedded in the rankings obtained with different information retrieval approaches, and it can help to find hard topics in the information retrieval task.
Figure \ref{fig:QueryK1K2mean} shows the 6-support (i.e., $ \frac{N}{2}$-support) consensus scores without weightings for the ranking sets of the 66 queries obtained by the 12 teams. The relative values of the consensus scores are generally consistent with the results in Figures \ref{fig:TausefdefCLEF} and \ref{fig:SpearmansefdefCLEF}.   However, our results based on $q$-support patterns, especially the pairwise patterns, reveal more obvious and detailed information. It can be seen from Figure \ref{fig:QueryK1mean} that the consensus score $\bar{\kappa}_{1}(6)>0.5$ for queries 10, 13, 15, 20, 24, 25, 31, 38, 57, 58, 59, 67. This means that, on average, more than $50\%$ of the ranked items in a ranking for these queries are emmbedded in at least half of the ranking set. When the orders of these ranked items are further considered,  Figure \ref{fig:QueryK2mean} shows that, on average, more than $15\%$ of the pairwise patterns of a ranking are supported by at least half of the rankings for queries 20, 24, 25, 38, 57, 58, 59, 67.
Figure \ref{fig:QueryK1K2meanWeighted} shows the 6-support consensus scores with the weighting parameters on the deviations of positions and position gaps being $\gamma = 0.9, \lambda =0.9$. It can be noticed that queries 58, 25, 24, 55 have higher consensus score $\bar{\kappa}_{2}(6)$, which indicates that the rankings of these queries share more weighted pairwise $q$-support patterns. Moreover, the consensus score $\bar{\kappa}_{1}(6)$ for these queries are also high. In contrast, the consensus scores of queries 64, 48, 11, 33 are much lower. The detailed information of these queries is given in Table \ref{table:QueryHighK1K2mean} and Table \ref{table:QueryLowK1K2mean}. By comparing the two tables, it can be found that the queries with clear descriptions or for typical symptoms tend to have higher consensus scores, while vague descriptions or uncommon symptoms lead to retrieval results with  lower consensus scores. 
\begin{figure}[!t]
    \centering
    \subfigure[Consensus score $\bar{\kappa}_{1}(6)$]
    {
        \includegraphics[width=3.2in]{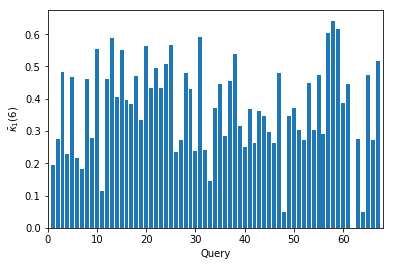}
        \label{fig:QueryK1mean}
    }
    \subfigure[Consensus score $\bar{\kappa}_{2}(6)$]
    {
        \includegraphics[width=3.2in]{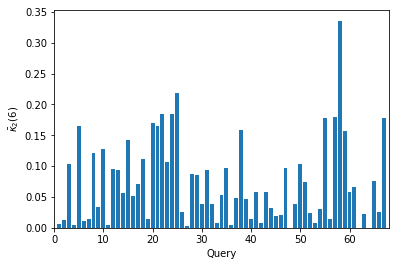}
        \label{fig:QueryK2mean}
    }
    \caption{Consensus scores without weighting for the ranking sets of the 66 queries obtained by the 12 teams}
    \label{fig:QueryK1K2mean}
\end{figure}

\begin{figure}[!t]
    \centering
    \subfigure[Consensus score $\bar{\kappa}_{1}(6)$]
    {
        \includegraphics[width=3.2in]{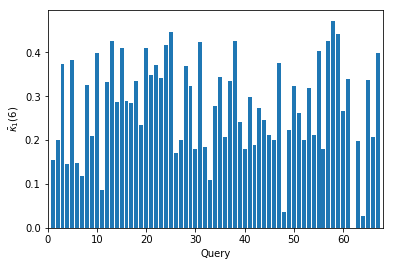}
        \label{fig:QueryK1meanWeighted}
    }
    \subfigure[Consensus score $\bar{\kappa}_{2}(6)$]
    {
        \includegraphics[width=3.2in]{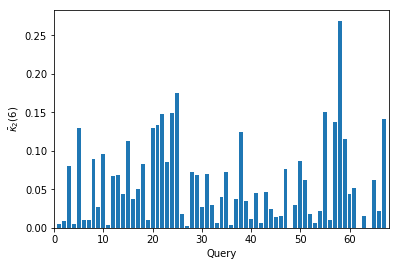}
        \label{fig:QueryK2meanWeighted}
    }
    \caption{Weighted consensus scores for the ranking sets of the 66 queries obtained by the 12 teams}
    \label{fig:QueryK1K2meanWeighted}
\end{figure}

\begin{table}[!h] 
\centering
\caption{Queries with higher consensus scores} 
\label{table:QueryHighK1K2mean}
			\begin{tabular}{|c|l|c|c|}
				\hline
				 Query ID &~~~~~~~~~~~~~~~~~ Query & $\bar{\kappa}_{1}(6)$ & $\bar{\kappa}_{2}(6)$ \\
				  \hline
				  58  & 39 degree and chicken pox & 0.47 & 0.27 \\
				  \hline
				  25 & red rash baby face &  0.45 & 0.17 \\
				  \hline
				  24 & yellow gunk coming from one eye itchy &  0.42 & 0.15\\
				  \hline
				  55 & crate type mark in skin &  0.40 & 0.15 \\
				  \hline
			\end{tabular}
\end{table}

\begin{table}[!h] 
\centering
\caption{Queries with lower consensus scores} 
\label{table:QueryLowK1K2mean}
			\begin{tabular}{|c|l|c|c|}
				\hline
				 Query ID & ~~~~~~~~~~~~~~~~~~~~~Query & $\bar{\kappa}_{1}(6)$ & $\bar{\kappa}_{2}(6)$ \\
				  \hline
				  64  & involuntary rapid left-right eye motion &  0.03 &  0.00 \\
				  \hline
				  48 & cannot stop moving my eyes medical condition & 0.04 & 0.00\\
				  \hline
				  11 & white patchiness in mouth & 0.09 & 0.00 \\
				  \hline
				  33 & white infection in pharynx & 0.11 & 0.01\\
				  \hline
			\end{tabular}
\end{table}

The  consensus of the information retrieval results for each topic is also evaluated with the proposed approch. The queries for each topic are supposed to link to an identical medical conditions. The consensus based on $2$-support patterns is studied for  the 22 topics each with three queries. Topic 13 is not considered, since it associates with query 62 having incomplete record in the dataset. We take the average of the consensus scores of the ranking sets of the 12 teams. The results are given in Figure \ref{fig:TopicK1K2mean}. 
Specially, the rankings of topics 15 and  11 have the highest average consensus scores, and the average consensus scores for topic 21 and topic 18 are the lowest. By comparing the topics and the details of the related queries in Table \ref{table:TopicHighK1K2mean} and Table \ref{table:TopicLowK1K2mean}, it can be found that the diseases of topics 15 and 11 are more common diseases to be easily self-diagnosed and the generated queries share more commonalities. On the contrary, the topics with low consensus scores have more diverse queries, thus they can be regard as hard topics, which can be used in further tasks for the development of more advanced search engines.  

\begin{figure}[!t]
    \centering
	\subfigure[Team average of the consensus score $\bar{\kappa}_{1}(2)$]
	{
		\includegraphics[width=3.2in]{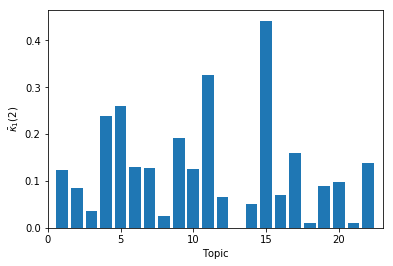}
		\label{fig:TopicK1mean}
	}
	\subfigure[Team average of the consensus score $\bar{\kappa}_{2}(2)$]
	{
		\includegraphics[width=3.2in]{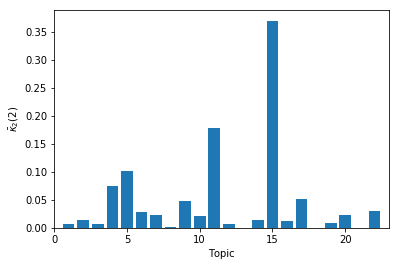}
		\label{fig:TopicK2mean}
	}
    \caption{Average weighted consensus scores for the 22 topics of the 3 queries}
    \label{fig:TopicK1K2mean}
\end{figure}


\begin{table}[!h] 
\centering
\caption{Topics with higher average consensus scores} 
\label{table:TopicHighK1K2mean}
			\begin{tabular}{|c|l|c|c|}
				\hline
				 Topic & ~~~~~~~~~~~~~~~~~~~~~~~~~~~~~Query & $\bar{\kappa}_{1}(2)$ & $\bar{\kappa}_{2}(2)$ \\
				  \hline
				  15: whooping cough & 12: baby has dry cough and has problem to swallow saliva  & 0.44 & 0.37 \\
				   (pertussis) & 46: baby cough &  &   \\
				   &  66: treatment of coughs in babies&  &  \\
				  \hline
				  11: bronchiolitis  & 31: toddler having squeaky breath& 0.32 & 0.17 \\
				   (caused by rsv) & 49: baby always  breathing with mouth closed &  & \\
				   &  59: heavy and squeaky breath &  & \\
				  \hline
			\end{tabular}
\end{table}


\begin{table}[!h] 
\centering
\caption{Topics with lower average consensus scores} 
\label{table:TopicLowK1K2mean}
			\begin{tabular}{|c|l|c|c|}
				\hline
				 Topic & ~~~~~~~~~~~~~~~~~~~~~~~~  Query & $\bar{\kappa}_{1}(2)$ & $\bar{\kappa}_{2}(2)$ \\
				  \hline
				  21: nystagmus & 36: eye are shaking & 0.01  &  0.00 \\
				    & 48: cannot stop moving my eyes medical condition  &  &  \\
				   & 64: involuntary rapid left-right eye motion&  &  \\
				  \hline
				  18: asthma wheezing & 6: child make hissing sound when breathing & 0.01   & 0.00 \\
				  & 15: asthma attack &  & \\
				   & 30: weird sounds when breathing  &  & \\
				  \hline
			\end{tabular}
\end{table}

\section{Conclusion}
\label{section: Conclusion}
This paper presents a novel approach to quantifying the consensus degree of a ranking set. A new concept of $q$-support has been introduced to represent the common patterns embedded in the rankings. A matrix representation has been developed to describe the commonality
within a ranking set that is shared by an individual ranking, on the basis of which an algorithm has been developed to quantify the  consensus efficiently. Moreover, a scheme for detecting outliers in a ranking set is derived from the consensus quantifying approach. Consensus evaluation with weighting on item positions and position gaps has also been considered. Compared with the existing methods based on correlation or distance functions, our approach can characterize and quantify the group preferences more explicitly and it also lays the foundation for the effective
detection of outliers and the development of rank aggregation algorithm, which have been illustrated in the experimental studies.

\ifCLASSOPTIONcaptionsoff
  \newpage
\fi

\bibliographystyle{IEEEtran} 
\bibliography{mybibliography}

\end{document}